\documentclass[]{article}
\usepackage{amsmath}
\usepackage{amsfonts}
\usepackage{amssymb}
\usepackage{amsthm}
\usepackage{mathrsfs}
\usepackage{enumerate}
\usepackage{xspace}
\usepackage[hidelinks]{hyperref} % The hidelinks is to get rid of borders around refernces.
\usepackage{xcolor} % for the \merk command I define below (copying Doerr)

\usepackage[backend=bibtex]{biblatex}% https://tex.stackexchange.com/questions/306218/simple-way-to-use-bibtex

\newtheorem{theorem}{Theorem}

\newtheorem{lemma}[theorem]{Lemma}

\DeclareMathOperator*{\argmin}{arg\,min}

\newcommand{\oea}{\mbox{$(1 + 1)$~EA}\xspace}

\newcommand{\onemax}{\textsc{OneMax}\xspace}
\newcommand{\hottopic}{\textsc{HotTopic}\xspace}
\newcommand{\zeromax}{\textsc{ZeroMax}\xspace}
\newcommand{\LO}{\textsc{Leading\-Ones}\xspace}
\newcommand{\leadingones}{\LO}

\newcommand{\blockleadingones}{\textsc{Block\-LeadingOnes}\xspace}

\newcommand{\jump}{\textsc{Jump}\xspace}

\newcommand{\hillpathjumpak}{\ensuremath{\textsc{HillPathJump}_{a,k}}\xspace}
\newcommand{\hillpathjump}{\ensuremath{\textsc{HillPathJump}}\xspace}

\newcommand{\sspace}{\ensuremath{{\{0,1\}^n\xspace}}}

\newcommand{\zeros}{\ensuremath{\mathbf{0}^n}\xspace}
\newcommand{\pak}{\ensuremath{P_{a,k}}\xspace}
\newcommand{\sqrtn}{\ensuremath{\sqrt{n}}\xspace}

\newcommand{\eps}{\varepsilon}

\DeclareMathOperator{\expectation}{E}

\title{All Mutation Rates $c/n$ for the \\ $(1+1)$ Evolutionary Algorithm}
\author{Andrew James Kelley}
\begin{document}

\maketitle

\begin{abstract}
	For every real number $c \geq 1$ and for all $\varepsilon > 0$, 
	there is a fitness function $f : \{0,1\}^n \to \mathbb{R}$ for which the 
	optimal mutation rate for the $(1+1)$ evolutionary algorithm on $f$, denoted $p_n$, 
	satisfies $p_n \approx c/n$ in that $|np_n - c| < \varepsilon$. In other words, the set of all $c \geq 1$ for which
	the mutation rate $c/n$ is optimal for the $(1+1)$ EA is dense in the interval $[1, \infty)$.
	To show this, a fitness function is introduced which is called
	 \textsc{HillPathJump}.
	 \end{abstract}
\section{Introduction}
In the last 25 years, many theorems have been proved about evolutionary algorithms (EAs). These results have helped
researchers make EAs that are better at optimizing relevant functions.

In some cases, we now know asymptotics of the expected runtime of certain EAs when using a particular mutation rate.
Sometimes, the optimal static mutation rate (i.e., the one that minimizes the expected runtime), is known. One might ask
what the set of possible optimal mutation rates is.

Recently, the author showed in \cite{Kelley26} that for all constant mutation rates $p = \Theta(1)$ and for all $\eps > 0$, there exists a fitness function for which
the optimal static mutation rate for the \oea is within $\eps$ of $p$, and this even includes all mutation
rates that are larger than 1/2.

For related work, note that for the \oea, it often is suggested that $1/n$ is the mutation rate to use; see \cite{DoerrLMN17}. We know that $1/n$ is
the optimal for the \oea on \onemax, and this mutation rate is also optimal for the $(1+\lambda)$ EA, as long as $\lambda$ isn't too large 
(see \cite{GiessenW17}).

For functions with local optima, it is known that the mutation rate can be larger than $1/n$. Indeed, it is shown in \cite{DoerrLMN17} 
that for $\jump_m$, the best mutation rate to use is $m/n$, where it is assumed $m = o(n)$. Also, \cite{JansenW00} introduces
a fitness function for which a mutation rate of $\Theta(\log n /n)$ is optimal for the \oea.

Note that for an optimal mutation rate of $c/n$ were $c$ is not an integer, both \leadingones (see \cite{BottcherDN10}) and the 
generalization \blockleadingones (see \cite{doerrK2023fourier}) have
as their optimal mutation rate for the \oea to be $c/n$, where $c$ is defined to be $c = \argmin_{x > 0} (e^x-1)/x^2$, and so the optimal
mutation rate here is approximately $1.5936/n$.
%See \cite{BottcherDN10} for \leadingones and \cite{doerrK2023fourier} for \blockleadingones.

It is natural to ask if there are other mutation rates $c/n$, with $c$ not an integer, for which $c/n$ is optimal for the \oea. This paper
shows that the set of all $c$ for which $c/n$ is the optimal mutation rate is dense in the infinite interval $[1, \infty)$.
To do this, the fitness function \hillpathjumpak is introduced.

As for mutation rates $c/n$ that are optimal with $c < 1$, note that it is possible that the optimal mutation rate for the 
\hottopic fitness functions have $c < 1$. See \cite{Lengler18} and \cite{LenglerS18}. In \cite{LenglerS18}, it was shown that
for every constant $c \geq 2.2$, there is a strictly monotone function such that with high probability, the runtime of the 
\oea with mutation rate $c/n$ is $e^{\Omega(n)}$. However, as shown in \cite{DoerrJSWZ10}, this contrasts with the runtime 
of $O(n \log n)$ for when $c \in (0, 1)$.
The result on $c \geq 2.2$ is a strengthening of a result from \cite{DoerrJSWZ13},
where there, $c > 16$.

The rest of this paper is as follows. Section~\ref{sec:definition_of_hillpathjump} defines \hillpathjump. 
Section~\ref{sec:results} lists the main results, and Section~\ref{sec:proofs} proves the main result(s).

\section{Definition of \hillpathjump}
\label{sec:definition_of_hillpathjump}
In this paper, for simplicity, we assume $n$ is a perfect square so that $\sqrt{n}$ is an integer.  

This section defines the \hillpathjump fitness functions. From lowest fitness to highest fitness, the first part is a large
``hill'' which is like \onemax except
that it is instead maximizing the number of zeros and so is called \zeromax. Afterwards is a path that leads to a jump.
The jump is of $k$ bits to the highest fitness individual, where $k$ is a constant (not depending on $n$).

%\merk{The change suggested here is no longer needed, I think: I'm changing the beginning part from $\sqrtn$ to $0.6n$}
For any $a > 0$ and integer $k \geq 4$, we will need to construct a path $\pak$ of length $a n^{k-1}$ starting from \zeros\ (the $n$-bit all-zeros string) and 
most of which is reasonably far from \zeros. The hamming distance between adjacent elements in the path is 1. 
The first $\sqrtn$ terms of $\pak$ is the sequence of points $z_1, z_2, \ldots, z_{\sqrtn}$,
where $z_i = 1^i0^{n-i}$ for $i = 1, 2\ldots, \sqrtn$.
(We assume that $n$ is large enough so that $an^{k-1} > \sqrtn$.) 

The rest of the path needs to be spread apart from itself and from \zeros. We do this by expanding a gray code by replacing 
each bit in a gray code with $\sqrtn$ bits.
Indeed, take a gray code on $(n - \sqrtn)/\sqrtn$ bits and then expand each bit to be $\sqrtn$ bits
(repeated).
As an example, let $n = 16$. Consider the gray code on 3 bits (i.e., on (16-4)/4 bits):
000, 001, 011, 010, 110, 111, 101, 100. Then, replace each bit with that bit repeated 4 times,
and expand the gray code path by replacing each transition 0 to 1 with 0000, 0001, 0011, 0111, 1111,
and do the opposite for transitions from 1 to 0. Expanding the above $2^3$ elements of the
gray code gives $4\cdot 2^3 - 3$ elements (left to right, top to bottom with spaces added for
readability); see Table~1.

\begin{table}[]
\begin{tabular}{llll}
 0000 0000 0000,& 0000 0000 0001, & 0000 0000 0011, & 0000 0000 0111, \\
 0000 0000 1111,&  0000 0001 1111, &  0000 0011 1111, &  0000 0111 1111, \\
 0000 1111 1111, & 0000 1111 0111, & 0000 1111 0011, & 0000 1111 0001, \\
 0000 1111 0000, & 0001 1111 0000,& 0011 1111 0000,& 0111 1111 0000,\\
 1111 1111 0000, & 1111 1111 0001, & 1111 1111 0011, & 1111 1111 0111,\\
 1111 1111 1111, &1111 0111 1111, &1111 0011 1111, &1111 0001 1111, \\
 1111 0000 1111, & 1111 0000 0111,&1111 0000 0011, & 1111 0000 0001,\\
 1111 0000 0000 & & & \\
\end{tabular}
\caption{Expanded gray code. The left-most column is the gray code with each
bit repeated 4 times. This defines a path from left to right, top to bottom.}
\end{table}

%The left-most column represents the expanded gray code, where each bit gets
%replaced with it repeated 4 times. To get to the next expanded gray code element,
%you just flip one bit at a time (going along the row to the right).

In general, first define a gray code path on $N = (n - \sqrtn)/\sqrtn$ bits,
which has $2^N$ elements in it: $y_1, y_2, \ldots, y_{2^N}$. Next, repeat
each bit in each $y_i$ a total of $\sqrtn$ times to form a path in $\{0, 1\}^{N\cdot \sqrtn}$
where adjacent elements have a hamming distance of $\sqrtn$. Next,
fill in the gaps so a transition from $0^{\sqrtn}$ to $1^{\sqrtn}$ happens
via $0^{\sqrtn - i}1^i$ for $i = 1,2,\ldots,\sqrtn$. Similarly, fill in a transition
from $1^{\sqrtn}$ to $0^{\sqrtn}$ via $0^{i}1^{\sqrtn - i}$ for 
$i = 1,2,\ldots,\sqrtn$. Doing this creates a path $\tilde{x}_1, \tilde{x}_2, \ldots, \tilde{x}_{M}$  in $\{0, 1\}^{N\cdot \sqrtn}$ with $M = \Theta(\sqrtn 2^N)$ elements.

To get the path $\pak = z_1, z_2, \ldots, z_{an^{k-1}}$ of $a n^{k-1}$ elements (where for simplicity\footnote{The asymptotic runtime of the \oea on \hillpathjump 
is not sensitive to changes in the path length being made larger or smaller by any constant number of elements. See the proof of Lemma~\ref{lem:time_to_optimize_path},
where a length of $O(\sqrt{n})$ elements is ignored.} by $a n^{k-1}$ we always mean $\lfloor a n^{k-1} \rfloor$), 
we already showed how to get the first $\sqrtn$ elements, namely $1^i0^{n-i}$ for $i = 1, 2\ldots, \sqrtn$.
To get the next  $a n^{k-1} - \sqrtn$ elements, just take that many elements of $\tilde{x}_i$ and
prefix each with $\sqrtn$ consecutive 1s. (The elements $\tilde{x}_i$ for $i > a n^{k-1} - \sqrtn$ are not used except for one more as the endpoint of a jump.)

Finally, the optimal point of the fitness function being defined is a jump of $k$ bits away
from the last point on the path $\pak$. Let $t = a n^{k-1} - \sqrtn + k$, and consider the
 point $\tilde{x}_t \in \{0, 1\}^{N\cdot \sqrtn}$, and prefix it by 
$\sqrtn$ consecutive 1s. Then this point, denoted $x^*$, is
the optimum.

We are ready to define \hillpathjump. 
For $x \in \sspace$, let \zeromax(x) be the number of 0s in $x$. In other words,
$\zeromax(x) = \sum_{i=1}^n (1 - x_i)$; here $x_i$ is the $i$th bit of $x$. Then we say
\[
  \hillpathjumpak(x) = 
  \begin{cases}
     \zeromax(x), &\text{if } x \notin \pak, \text{ and } x \neq x^* \\
     n + i &\text{if } x = z_i, \text{ and} \\
     n + a n^{k-1} + 1 &\text{if } x = x^*.
  \end{cases}
\]

We need a lemma to justify the claim that there is indeed a jump of $k$ bits from the second highest fitness individual to $x^*$:
\begin{lemma}
  Let $x^+$ be the individual with second highest fitness. 
  Then the hamming distance between $x^+$ and $x^*$, denoted $H(x^+, x^*)$, is $k$.
\end{lemma}
\begin{proof}
  This follows because $k$ is a constant, and so we may assume that $k < \sqrtn$. Indeed,
  denote the expanded gray code path constructed as $\tilde{x}_1, \tilde{x}_2, \ldots, \tilde{x}_{M}$, where  for all $i$, we have
  $\tilde{x}_i \in \{0, 1\}^{N\cdot \sqrtn}$.  
  
  Let $a \in [1..M-k]$.  Because $k < \sqrt{x}$, we have  $\tilde{x}_a$ and $\tilde{x}_{a+k}$ are in the expanded bits version
  of either the same gray code bitstring or adjacent gray code bitstrings. Therefore, the
  hamming distance $H(\tilde{x}_a, \tilde{x}_{a+k})$ equals $k$.
\end{proof}

\section{Main Results}
\label{sec:results}

\begin{theorem}
\label{thm:main_result}
  Let $k \geq 4$ be an integer, and let $c \in (1, k)$ be a real number.
  Then there exists an $a \in (0, \infty)$
  such that the optimal mutation rate $p_n$ for the \oea on \hillpathjumpak satisfies $p_n \sim c/n$ (i.e., $p_n /(c/n) \to 1$ as $n \to \infty$).
\end{theorem}

The following theorem is an immediate consequence of Theorem~\ref{thm:main_result}.

\begin{theorem}
	For every real number $c \geq 1$ and for all $\varepsilon > 0$, 
	there is a fitness function $f : \{0,1\}^n \to \mathbb{R}$ for which the 
	optimal mutation rate for the $(1+1)$ evolutionary algorithm on $f$, denoted $p_n$, 
	satisfies $p_n \approx c/n$ in that $|np_n - c| < \varepsilon$.
\end{theorem}

\section{Proof of Theorem~\ref{thm:main_result}}
\label{sec:proofs}

\begin{lemma}
\label{lem:time_to_optimize_path}
 Suppose a mutation rate of $c/n$ is used.
 Let $x^+$ be the individual of second-highest fitness.
  Let $T_2$ be the time to find either $x^+$ or the 
  optimum $x^*$. Then $\expectation[T_2] = (a/c) e^c n^{k}(1 + o(1))$.
\end{lemma}
\begin{proof}
  Theorem 7 from \cite{DoerrG13algo} implies that it takes $O(n \log n)$ to optimize \zeromax. Therefore,
  on \hillpathjump, it takes $O(n \log n)$ time to either optimize the (modified) \zeromax or to find $x^*$ or an element in $\pak$.
  Let $F$ be the fitness of the first individual found in $\pak \cup \{x^*\}$. Then $\expectation[F] \leq n+ \sqrtn$.
  Next, the path to $x^+$
  has length $a n^{k-1} - (F-n)$. 
  Let $T_2'$ be the number of steps to reach $\{x^+, x^*\}$ once a point on $\pak \cup \{x^*\}$ is found.
  So $T_2 = T_2' + O(n \log n)$. 
  Using additive drift (Theorem 2.3.1 from \cite{DoerrN20}), we see that
  by Lemma~\ref{lem:additive_drift},
  it takes expected time $\expectation[T_2]$ to find $x^+$ (or the optimum $x^*$) where
  \[
    \begin{aligned}
       \expectation[T_2' ] &= \frac{a n^{k-1} - \expectation[F-n]}{\frac{c e^{-c}}{n}(1 + o(1))}\\
        &= \frac{a n^{k-1} - O(\sqrtn)}{\frac{c e^{-c}}{n}(1 + o(1))}\\
              &= (a/c) e^c n^{k}(1 + o(1)).
    \end{aligned}
  \]
  Then $T_2 = T_2' + O(n \log n) = (a/c) e^c n^{k}(1 + o(1))$.
  %where adding $O(n \log n)$ to this does not change the asymptotics.
\end{proof}

\begin{lemma}
\label{lem:additive_drift}
  While the individual is on the path $\pak = z_1, z_2,\ldots, z_L$, define $X_t$ be $L-j$ where
  $j$ is such that the individual in generation $t$ is $z_j$. Then for
  $\Delta_t(s) := \expectation[X_t - X_{t+1} \mid X_t = s]$ with $s > 0$, we have additive drift:
  \[
    \Delta_t(s) = \frac{c e^{-c}}{n}(1 + o(1)).
  \]
\end{lemma}
\begin{proof}
  Assuming $X_t = s$, a jump to the optimum has probability 
 % \[
  %(an^{k-1}-s)\left( \frac{c}{n} \right)^{an^{k-1}-s + k}\left(1 - \frac{c}{n} \right)^{n - an^{k-1} + s - k},
  %\]
   bounded by $O(1/n^k)$.
  Thus, we have
  \[
    \begin{aligned}
    \Delta_t(s) &= 1 \frac{c}{n}\left(1 - \frac{c}{n} \right)^{n-1} + \sum_{j=2}^{an^{k-1}-s}j \left(\frac{c}{n}\right)^j \left(1 - \frac{c}{n} \right)^{n-j} + 
    				O\left( (1/n)^k \right) \\
		     &\leq \frac{c}{n}\left(1 - \frac{c}{n} \right)^{n-1} + \sum_{j=2}^{an^{k-1}-s} j O((c/n)^j) + O\left( (1/n)^k \right) \\
    		     &= \frac{c}{n}\left(1 - \frac{c}{n} \right)^{n-1} + O\left(\frac{1}{n^2}\right)  \\
    		     &= \frac{c e^{-c}}{n}(1 + o(1))
    \end{aligned}
  \]
\end{proof}

It is clear that there is a high probability that the EA reaches the second-highest fitness individual before
the optimum $x^*$. Here is a more precise statement:
\begin{lemma}
\label{lem:probability_of_no_early_jump}
Assume $k \geq 4$.
Let $x^+$ be the individual of second-highest fitness.
  The probability of hitting $x^*$ before
  $x^+$ is $O(1/n^2)$.
\end{lemma}
\begin{proof}
  Recall $\pak = z_1, z_2,\ldots, z_{an^{k-1}}$. For $i \leq an^{k-1} - \sqrtn$ the hamming distance from $z_i$ to $x^*$ satisfies 
  $H(z_i, x^*) \geq \sqrtn$. For such $i$, the probability to go from $z_i$ to $x^*$ in one step is bounded by $(c/n)^{\sqrtn}$.
  By Lemma~\ref{lem:time_to_optimize_path}, the expected time to stay on the path, which we'll denote $\expectation[T_1]$, is bounded by
  $(a/c) e^c n^{k}(1 + o(1))$. Let $\eps > 0$, and let $Q_1$ denote the event that $T_1 \geq (a/c) e^c n^{k+2}(1 + \eps)$. By Markov's inequality, 
  \[
    \Pr(Q_1) = O(1/n^2).
  \]
  Let $J_1$ be the event that the EA jumps from $z_i$ to $x^*$ for $i \leq an^{k-1} - \sqrtn$.
  Then
  \[
    \begin{aligned}
     \Pr(J_1) &= \Pr(J_1 \mid Q_1) \Pr(Q_1) + \Pr(J_1 \mid \overline{Q}_1) \Pr(\overline{Q}_1) \\
     &\leq 1 \cdot \Pr(Q_1) + \Pr(J_1 \mid \overline{Q}_1) \cdot 1 \\
     &\leq  O(1/n^2) + \sum_{j=1}^{(a/c) e^c n^{k+2}(1 + \eps)} (c/n)^{\sqrtn}.
    \end{aligned}
  \]
  where the last inequality uses the union bound. Note that 
  \[
  \begin{aligned}
  \sum_{j=1}^{(a/c) e^c n^{k+2}(1 + \eps)} (c/n)^{\sqrtn} &= (a/c) e^c n^{k+2}(c/n)^{\sqrtn}(1 + \eps) \\
  &= O(1/n^2).
  \end{aligned}
  \]  
  %$\sum_{i=1}^{(a/c) e^c n^{k+2}(1 + o(1))} (c/n)^{\sqrtn} \leq (a/c) e^c n^{k+2}(c/n)^{\sqrtn}(1 + o(1)) \leq 1/n^2$.
  Thus $\Pr(J_1) = O(1/n^2)$.
  
  Let $J_2$ be the event that the EA jumps from $z_i$ to $x^*$ for any $i$ such that $an^{k-1} - \sqrtn < i \leq an^{k-1}-2$. 
  Let $T_{p}'$ be the time that the current individual is any $z_i$ for such $i$. Then by Lemma~\ref{lem:additive_drift},
  $\expectation[T_2] \leq \sqrtn/[ce^{-c}/n(1+o(1))] = O(n^{1.5})$. For these $i$, the hamming distance satisfies 
  $H(z_i, x^*) \geq k+2$. Let $Q_2$ denote the event that $T_2 \geq n^{3.5}$.
   By Markov's inequality, $\Pr(Q_2) = O(1/n^2)$.
   Then
  \[
   \begin{aligned}
     \Pr(J_2) &= \Pr(J_2 \mid Q_2) \Pr(Q_2) + \Pr(J_2 \mid \overline{Q}_2) \Pr(\overline{Q}_2) \\
     &\leq 1 \cdot O(1/n^2) + \Pr(J_2 \mid \overline{Q}_2) \cdot 1 \\
     &\leq O(1/n^2) + \sum_{j=1}^{n^{3.5}} \frac{1}{n^{k+2}} \\
     &\leq O(1/n^2) + O(1/n^{k-1.5}),
   \end{aligned}
  \]
  which is bounded by $O(1/n^2)$, since $k \geq 4$.
  
  Let $J_3$ be the event that the EA jumps from $z_i$ to $x^*$ for  $i = an^{k-1}-1$. Let $T_3$ be the time spent
  at this one point. Then $T_3$ is a geometric random variable with probability of success at least $c/n(1-c/n)^{n-1} = \Theta(1/n)$.
  Since $T_3$ is a geometric random variable, 
  $\expectation[T_p] = \Theta(n)$.  Also, the hamming distance satisfies $H(z_i, x^*) = k+1$. 
  Let $Q_3$ denote the event that $T_3 \geq n^3$. By Markov's inequality, $\Pr(Q_3) = O(1/n^2)$. Then
  \[
    \begin{aligned}
     \Pr(J_3) &= \Pr(J_3 \mid Q_3) \Pr(Q_3) + \Pr(J_3 \mid \overline{Q}_3) \Pr(\overline{Q}_3) \\
     &\leq 1 \cdot O(1/n^2) + \Pr(J_3 \mid \overline{Q}_3) \cdot 1 \\
     &\leq O(1/n^2) + \sum_{j=1}^{n^{3}} \frac{1}{n^{k+1}} \\
     &\leq O(1/n^2) + O(1/n^{k-2}),
   \end{aligned}
  \]
  which is bounded by $O(1/n^2)$, since $k \geq 4$.
  
  Altogether, the union bound gives $\Pr(J_1 \cup J_1 \cup J_3)  = O(1/n^2)$. Adding to this the exponentially small
  probability that the EA first finds $x^*$ before finding any other point on \pak doesn't change this bound.
\end{proof}

\begin{lemma}
\label{lem:sec_best_to_optimum}
  Let $T_1$ be the number of further steps it takes to visit $x^*$ given that
  the current individual is $x^+$. Then $\expectation[T_1] = (e^c/c^k) n^k(1+o(1)).$
\end{lemma}
\begin{proof}
 The variable $T_1$ is a geometric random variable with probability of success 
 \[
   \left(\frac{c}{n}\right)^k\left(1 - \frac{c}{n} \right)^{n-k},
 \]
 which equals $c^k e^{-c}/n^k (1+o(1))$. We done by taking the reciprocal.
\end{proof}

\begin{lemma}
\label{lem:optimization_time}
  Assume $k \geq 4$. Let $T$ be the time for the \oea to optimize \hillpathjumpak using a mutation rate
  of $c/n$. Then
  \[
    \expectation[T] = \left(\frac{a}{c} + \frac{1}{c^k}\right) e^c n^k (1 + o(1)).
  \]
\end{lemma}
\begin{proof}
  That $\left(\frac{a}{c} + \frac{1}{c^k}\right) e^c n^k (1 + o(1))$ is an upper bound for $\expectation[T]$ follows by adding
  the expected times in Lemmas~\ref{lem:time_to_optimize_path} and \ref{lem:sec_best_to_optimum}. That this 
  expression works as a lower bound follows from Lemma~\ref{lem:probability_of_no_early_jump} and multiplying the
  bound by $1 - O(1/n^2)$.
  %%%
\end{proof}

\begin{lemma}
\label{lem:opt_mute_rates_are_argmin_of_gx}
  For $a > 0$, let $g(x) = g_a(x) = (a/x + 1/x^k)e^x$ with domain restricted to positive $x$.
  Let $c =  \argmin g_a(x)$. Then as $n \to \infty$, the optimal 
  optimal mutation rate for the \oea on
   \hillpathjumpak approaches $c/n$
%  \[
%     \{m \in \mathbb{R}_{> 0} \mid m =  \argmin g_a(x) \text{ for some } a \}.
 % \]
\end{lemma}
\begin{proof}
  This follows from Lemma~\ref{lem:optimization_time}. The reason we cannot say that the optimal mutation
  rate \emph{equals} $c/n$ is because of the asymptotic term $o(1)$, which forces us to make $n \to \infty$.
\end{proof}

\begin{lemma}
\label{lem:relating_zeros_to_argmin}
  Assume $k \geq 2$.
  Let $g(x)$ be as in Lemma~\ref{lem:opt_mute_rates_are_argmin_of_gx}, 
  and let $q(x) = a x^k - a x^{k-1} + x - k = ax^{k-1}(x - 1) + x -k$ 
  with domain restricted to positive $x$.
  Then $q(x)$ has a unique zero $z_0 \in (1, \infty)$, and in fact, $z_0 = \argmin g(x)$.
\end{lemma}
\begin{proof}
  Note that $q(x) < 0$ for $x \in (0, 1]$, and so $q(x) = 0$ implies $x > 1$. Next, 
  \[
    \begin{aligned}
       q'(x) &= a k x^{k-1} - a(k-1)x^{k-2} + 1 \\
               &= a x^{k-2}(k x - (k-1)) + 1.
    \end{aligned}
  \]
  Therefore, $q'(x) > 0$ for $x \geq 1$. Hence, $q(x)$ is increasing on $[1, \infty)$.
  Since, as we have seen, $q(1) < 0$, we may conclude that $q(x)$ has a unique
  zero in $(1, \infty)$ because $q(x)$ is eventually positive. 
  
  As we shall soon see, that the zero of $q(x)$ is the $\argmin$ of $g(x)$ follows from the fact
  that $x^{k+1}e^{-x}g'(x) = q(x)e^x$. Indeed, 
  \[
    g'(x) = \left(\frac{a}{x} + \frac{1}{x^k} - \frac{a}{x^2} - \frac{k}{x^{k+1}} \right) e^x.
  \]
  Any local extrema of $g(x)$ occur when $g'(x) = 0$, and $g'(x) = 0$ if and only if
  $x^{k+1}e^{-x}g'(x) = 0$. We are done by noting first that $x^{k+1}e^{-x}g'(x) = q(x)$
  and second that $q(x)$ goes from being negative to positive, which implies that
  the local extrema is the absolute minimum of $g(x)$.
\end{proof}

\begin{lemma}
\label{lem:zero_any_value_in_interval}
   Let $q(x) = q_a(x) = a x^k - a x^{k-1} + x - k = ax^{k-1}(x - 1) + x -k$, and let $Z(a)$ be the zero
   of $q_a(x)$. Then by appropriate choice of $a$, $Z(a)$ can be any number in $(1, k)$.
\end{lemma}
\begin{proof}
  First, note that $Z(a)$ is continuous as a function of $a$ (because the complex roots of
  a polynomial are a continuous (multi-valued) function of the coefficients). 
  
  Next, note that $\lim_{a \to \infty} Z(a) = 1$; indeed, this is because $q(1) = 1-k < 0$, and for all
  $x > 1$, we have $\lim_{a \to \infty} a x^{k-1}(x-1) = \infty$.
  
  Next, note that $\lim_{a \to 0} Z(a) = k$; indeed, this is due to the fact that if 
  $a x^{k-1}(x-1) \approx 0$, then $q(x) = 0$ near where $x - k = 0$, and this happens for $x = k$. 
  
  As has been said, $Z(a)$ is a continuous function of $a$, and $Z(a)$ can take on values
  arbitrarily close to 1 and to $k$. Therefore, it can take on any value in $(1, k)$.
\end{proof}

\begin{proof}[Proof of Theorem~\ref{thm:main_result}]
 This follows from Lemmas~\ref{lem:opt_mute_rates_are_argmin_of_gx}, \ref{lem:relating_zeros_to_argmin}, 
 and \ref{lem:zero_any_value_in_interval}.
  %%%
\end{proof}

%\section*{Acknowledgments}
%I would like to thank Benjamin Doerr for a helpful comment on an earlier version of this paper.

\printbibliography

\end{document}